\documentclass[a4paper]{article}
\usepackage{authblk}
\usepackage{bm}
\usepackage{enumitem,amsthm,amsmath,amsfonts,amssymb}
\usepackage{mathtools}
\usepackage[linesnumbered,ruled,vlined]{algorithm2e}
\usepackage{xcolor}
\usepackage{graphicx}
\usepackage{tikz}
\usepackage{multirow}
\usepackage{multicol}
\usepackage{hyperref}
\usepackage[numbers,sort&compress]{natbib}
\usepackage{subfigure}
\usepackage{mathrsfs, fancyhdr}
\usepackage{float}
\newtheorem{theorem}{Theorem}
\newtheorem{lemma}[theorem]{Lemma}

\newtheorem{corollary}[theorem]{Corollary}
\theoremstyle{definition}
\newtheorem{definition}{Definition}

\newtheorem{example}{Example}
\theoremstyle{definition}

\newtheorem{remark}{Remark}

\newcommand{\pr}{\mathrm{Pr}}

\newcommand{\nbb}{\mathbb{N}}
\newcommand{\bz}{\mathbf{z}}
\newcommand{\bw}{\mathbf{w}}

\newcommand{\ibb}{\mathbb{I}}
\newcommand{\xcal}{\mathcal{X}}
\newcommand{\wcal}{\mathcal{W}}

\newcommand{\bW}{\mathbf{W}}

\newcommand{\zcal}{\mathcal{Z}}

\newcommand{\ycal}{\mathcal{Y}}

\newcommand{\ebb}{\mathbb{E}}

\newcommand{\rbb}{\mathbb{R}}
\newcommand{\pbb}{\mathbb{P}}

\numberwithin{equation}{section}

\newcommand{\bx}{x}

\newcommand{\N}{\mathbb{N}}

\newcommand{\R}{\mathbb{R}}

 \newcommand{\ACe}{\color{black}}
\newlength{\fixboxwidth}
\title{Bootstrap SGD: Algorithmic Stability and Robustness}
\date{\today}
\author[1]{Andreas Christmann}
\author[2]{Yunwen Lei}

\affil[1]{Department of Mathematics, University of Bayreuth, 95440 Bayreuth, GERMANY\\ \url{andreas.christmann@uni-bayreuth.de}}
\affil[2]{Department of Mathematics,
The University of Hong Kong,
Pokfulam
Hong Kong, China\\ \url{leiyw@hku.hk} }

\begin{document}
\maketitle


\begin{quote}
In this paper some methods to use the empirical bootstrap approach for stochastic gradient descent (SGD) to minimize the empirical risk over a separable Hilbert space are investigated from the view point of algorithmic stability and statistical robustness. The first two types of approaches are based on averages and are investigated from a theoretical point of view. A generalization analysis for bootstrap SGD of Type 1 and Type 2 based on algorithmic stability is done. Another type of bootstrap SGD is proposed to demonstrate that it is possible to construct purely distribution-free pointwise confidence intervals of the median curve using bootstrap SGD.
\end{quote}

\section{Introduction}

Bootstrap is an effective statistical method to improve the generalization of machine learning models by resampling. The basic idea is to first use sampling with replacement to produce  $B$ bootstrap samples from the original training data set. Then, for the $b$-th bootstrap sample ($b=1,\ldots,B$), a learning algorithm is applied to produce a model $h^{(b)}$. Finally, an aggregation method is used to combine these local models for prediction.
There exists a huge literature on bootstrap, and we refer to \cite{Efron1979, Efron1982,EfronHastie2016} and the references therein.

Depending on the aggregation method, we have different types of bootstrap methods.
In this paper, we consider three types of bootstrap methods for stochastic gradient descent (SGD) methods, which become the workhorse behind the success of many machine learning applications and have received a lot of attention~\citep{neu2021information,orabona2019modern,hardt2016train}.
For bootstrap of Type 1, we take an average of weight parameters and use the model associated to the averaged weight parameter for prediction. For bootstrap of Type 2, we first use each local model for prediction, and then take an average of these predicted outputs as the final prediction. For bootstrap of Type 3, we also first predict by each local model but then use the  highly robust median of these predicted output values to obtain
pointwise confidence intervals for the median and pointwise tolerance intervals.

Due to the resampling and the aggregation scheme, bootstrap  can be  useful to improve the stability and robustness, which intuitively means the ability to withstand perturbations  and outliers, which are both common in many applications
and in routine data. We say an algorithm is stable if the output of the algorithm is not sensitive to the perturbation of a training dataset~\citep{bousquet2002stability,hardt2016train,shalev2010learnability,liu2017algorithmic}.  Pioneering stability analysis rigorously shows the improvement of stability by bootstraping~\citep{elisseeff2005stability}. However, their analysis treated multiple copies of a training example as a single example, which is not the choice in practice.
Furthermore, bootstrap is also useful to do inference such as building confidence intervals for some parameters of interest~\citep{Efron1979}. This inference is useful to understand how reliable the prediction is~\citep{ChristmannSalibianBarreraVanAelst2013,Hampel1974}.

In this paper, we study the algorithmic stability and robustness of bootstrap SGD.
Our main contributions are summarized as follows.
\begin{itemize}
  \item We study the $\ell_1$- and $\ell_2$-argument stability of bootstrap SGD applied to convex, smooth and Lipschitz problems. Our $\ell_2$-argument stability analysis shows how the bootstrap samples improve the generalization behavior. We consider bootstrap SGD of Type 1 and Type 2. For bootstrap SGD of Type 1, our analysis is able to handle multiple copies of a training example in a bootstrap sample, which was treated as a single example in the existing analysis~\citep{elisseeff2005stability}. We give the first stability analysis for bootstrap SGD of Type 2.
  \item  Using a well-known result from order statistics, one can
    easily construct pointwise confidence intervals for the median and
    pointwise tolerance intervals. From our point of view,
    this enlarges the options of bootstrap SGD. Although we do not investigate
    these intervals from a theoretical point of view, we show their usefulness
    in a numerical example.
\end{itemize}

The remainder of the paper is organized as follows. We introduce the background in Section~\ref{sec:background}. We study the stability of bootstrap SGD in Section~\ref{sec:stab}, and the distribution-free confidence intervals in Section~\ref{sec:confidence}. We present numerical analysis to illustrate the behavior of bootstrap SGD in Section~\ref{sec:example}, and conclude with a short discussion  in Section~\ref{sec:discussion}. The proofs are given in Section~\ref{sec:proof}.

%
%

\section{Background\label{sec:background}}

\subsection{Algorithmic Stability}
Let $\pbb$ be a probability measure defined over a sample space
$\zcal=\xcal\times\ycal$, where $\xcal$ is an input space and $\ycal$ is an output space.
Let $S= (z_1,\ldots,z_n)\in\zcal^{n}$ be a training dataset of size $n$, based on which we aim to build a function
$h:\xcal\mapsto\ycal$ for prediction.
We assume the prediction function $h$ is indexed by a parameter $\bw\in\wcal$, where $\wcal$ is a normed space and has the role of a parameter
space.
Let $\ell:\rbb\times\rbb\mapsto\rbb_+$ be a loss function, and we denote $\ell(h_{\bw}(x),y)$ the loss suffered by using $h_{\bw}$ to do prediction on $z=(x,y)$.
For brevity, we denote $f(\bw;z):=\ell(h_{\bw}(x),y)$. The behavior of a model on training and testing is then measured by the empirical and population risk as follows
\[
F_S(\bw):=\frac{1}{n}\sum_{i=1}^{n}f(\bw;z_i),\quad F(\bw):=\ebb_z[f(\bw;z)],
\]
where $\ebb_z[\cdot]$ denotes the expectation with respect to (w.r.t.) the distribution of $z$.
We often apply a (randomized) learning algorithm $A$ to (approximately) minimize the empirical risk, and we denote by $A(S)$ the model derived by applying $A$ to the dataset $S$.
We will assume in this paper -- if not otherwise mentioned -- that $\wcal$ is a separable Hilbert space and that the
(measurable)\footnote{We will always assume that the learning algorithm $A$ is measurable, i.e., for all
$n\ge 1$, the map $(\xcal\times\ycal)^{n}\times \xcal \to \R$, $(S,x) \mapsto A(S)(x)$ is measurable with respect to the universal completion of the product $\sigma$-algebra on $(\xcal\times\ycal)^{n}\times \xcal$,
where $A(S)(\cdot)$ denotes the decision function of the learning algorithm $A$.}  learning algorithm maps into a separable Hilbert space. Examples are
of course $\R^{d}$ or separable reproducing kernel Hilbert spaces.
The relative behavior of $A(S)$ as compared to the best model $\bw^*=\arg\min_{\bw\in\wcal}\; F(\bw)$ is referred to as the excess population risk $F(A(S))-F(\bw^*)$, which can be decomposed as follows
\[
F(A(S))-F(\bw^*)=F(A(S))-F_S(A(S))+F_S(A(S))-F_S(\bw^*)+F_S(\bw^*)-F(\bw^*).
\]
We refer to the first term $F(A(S))-F_S(A(S))$ and the third term $F_S(\bw^*)-F(\bw^*)$ as the generalization gap, as they measure the difference between training and testing. We refer to the second term $F_S(A(S))-F_S(\bw^*)$ as the optimization error as it measures the suboptimality of $A(S)$ as measured by the training error. The optimization error is a central concept in optimization theory and has been extensively studied in the literature~\citep{orabona2019modern}. The generalization gap is a central concept in statistical learning theory, which is closely related to the stability~\citep{kuzborskij2013stability,nikolakakisbeyond,schliserman2022stability,zhang2024generalization} and robustness~\citep{christmann2016robustness,huang2022learning,xu2012robustness,sun2024total} of the learning algorithm. In this paper, we will leverage the algorithmic stability to study the stability of boostrap algorithms. We first introduce several popular stability concepts.
We denote $S\sim\widetilde{S}$ if they are neighboring datasets, i.e., $S$ and $\widetilde{S}$ differ by a single example. Let $\epsilon>0$ and $\|\cdot\|_2$ denote the $\ell_2$ norm.
\begin{definition}[Uniform stability~\citep{bousquet2002stability}]
Let $A$ be a (randomized) learning algorithm. We say $A$ is uniformly stable with parameter $\epsilon$ if $\sup_{S\sim\widetilde{S}}\sup_z\ebb_A[f(A(S);z)-f(A(\widetilde{S});z)]\leq \epsilon$, where $\ebb_A[\cdot]$ denotes the expectation w.r.t. the distribution of the algorithm $A$.
\end{definition}

\begin{definition}[Argument stability]
Let $A$ be a (randomized) learning algorithm. We say $A$ is $\ell_1$-argument stable with parameter $\epsilon$ if $\sup_{S\sim\widetilde{S}}\ebb_A[\|A(S)-A(\widetilde{S})\|_2]\leq \epsilon$. We say $A$ is $\ell_2$-argument stability with parameter $\epsilon$ if $\sup_{S\sim\widetilde{S}}\ebb_A[\|A(S)-A(\widetilde{S})\|_2^2]\leq\epsilon^2$.
\end{definition}

It is clear that $\ell_2$-argument stability is stronger than $\ell_1$-argument stability, which further implies uniform stability if $f$ is Lipschitz continuous.

\begin{definition}
Let $g:\wcal\mapsto\rbb$, where $\wcal$ is a normed space. Let $L,G\geq0$.
\begin{itemize}
  \item We say $g$ is $G$-Lipschitz continuous if for any $\bw,\bw'\in\wcal$ we have $|g(\bw)-g(\bw')|\leq G\|\bw-\bw'\|_2$.
  \item We say $g$ is $L$-smooth if for any $\bw,\bw'\in\wcal$ we have $\|\nabla g(\bw)-\nabla g(\bw')\|_2\leq L\|\bw-\bw'\|_2$, where $\nabla$ denotes the gradient operator.
  \item We say $g$ is convex if for any $\bw,\bw'\in\wcal$ we have $g(\bw)\geq g(\bw')+\langle\bw-\bw',\nabla g(\bw')\rangle$, where $\langle\cdot,\cdot\rangle$ denotes the dot product.
\end{itemize}
\end{definition}

\begin{example}
Let $\xcal=\R^{d}$, $\ycal\subset \R$, and $z=(x,y)\in \xcal\times\ycal$.
\begin{enumerate}
\item[(i)] The least squares loss function $f=f_{\text{LS}}$ is defined by $f(\bw;z)=\frac{1}{2}(y-h_{\bw}(x))^{2}$.
\item[(ii)] The logistic loss function $f=f_{\text{c-logis}}$ for binary classification is defined by
$f(\bw;z)=\ln\bigl(1+\exp(-yh_{\bw}(x))\bigr)$, where $\ycal=\{-1,+1\}$.
\item[(iii)] The logistic loss function $f=f_{\text{r-logis}}$ for regression is defined by
$f(\bw;z)=-\ln\frac{4 e^{y-h_{\bw}(x)}}{(1+e^{y-h_{\bw}(x)})^{2}}$~\citep{steinwart2008support}.
\end{enumerate}

Basic computations show that
\begin{eqnarray*}
  \nabla f_{\text{LS}}(\bw;z) & = & (h_{\bw}(x)-y) \nabla h_{\bw}(x),\\
 \nabla f_{\text{c-logis}}(\bw;z) & = & -\frac{y}{1+\exp(y h_{\bw}(x))} \nabla h_{\bw}(x),\\
\nabla f_{\text{r-logis}}(\bw;z)& = & \tanh\bigl((h_{\bw}(x)-y)/2\bigr) \nabla h_{\bw}(x).
\end{eqnarray*}
Note, that these three loss functions are convex with respect to $h_{\bw}(x)$.
To clearly see the property of these loss functions, we consider a linear model $h_{\bw}(x)=\bw^\top x$ and assume $\|x\|_2\leq 1$ for all $x\in \xcal$ and $\bw^\top$ means the transpose of $\bw$. In this case, it can be directly checked that all these loss functions are smooth. Furthermore, the logistic loss for both regression and binary classification is Lipschitz continuous. For example, it is clear that
\[
\big\|\nabla f_{\text{c-logis}}(\bw;z)\big\|_2=\frac{|y|}{1+\exp(y h_{\bw}(x))} \cdot \big\|\nabla h_{\bw}(x)\big\|_2 \leq \|x\|_2\leq 1.
\]
However, the least square loss is not Lipschitz continuous since
\[
\|\nabla f_{\text{LS}}(\bw;z)\|_2=|h_{\bw}(x)-y| \cdot \|\nabla h_{\bw}(x)\|_2
= |\bw^\top x-y| \cdot \|x\|_2,
\]
which goes to infinity as the norm of $\bw$ goes to infinity.
\end{example}

The following lemma shows the quantitative connection between stability and generalization.
The first part shows the connection between generalization and $\ell_1$-argument stability, while the second part shows the connection between generalization and $\ell_2$-argument stability.
\begin{lemma}[Stability and Generalization\label{lem:stab-gen}~\citep{lei2020fine}]
Let $A$ be an algorithm. Let $G,L,\epsilon>0$.
\begin{itemize}
  \item Suppose $A$ is $\ell_1$-argument stable with parameter $\epsilon$.  If for any $z$, the map $\bw\mapsto f(\bw,z)$ is $G$-Lipschitz continuous, then $\ebb[F(A(S))-F_S(A(S))]\leq G\epsilon$.
  \item Suppose $A$ is $\ell_2$-argument stable with parameter $\epsilon$.  If for any $z$, the map $\bw\mapsto f(\bw,z)$ is $L$-smooth, then
  \[
  \ebb[F(A(S))-F_S(A(S))]\leq \frac{L\epsilon^2}{2}+\epsilon\big(2L \ebb[F_S(A(S))]\big)^{\frac{1}{2}}.
  \]
\end{itemize}
\end{lemma}


\subsection{Bootstrap SGD}
Bootstrap methods were introduced by \cite{Efron1979,Efron1982} and can be used to estimate
standard errors of estimators in a nonparametric manner.
There exists a huge literature on bootstrap methods, see e.g.
\cite{EfronHastie2016} and the references therein, and for specific bootstrap results for machine learning methods based on kernels we refer to \cite{ChristmannHable2013} and
\cite{ChristmannSalibianBarreraVanAelst2013}.
Bootstrapping can in particular be used
to construct distribution-free confidence intervals and distribution-free tolerance intervals.

Let $S=(\bz_1,\ldots,\bz_n)$. Bootstrap methods first build $B\in\nbb$ bootstrap samples by repeatedly drawing samples from $S$, each of which is of size $m$.
Let
$$
I_b=(i_{b,1},\ldots,i_{b,m}),\quad b=1,\ldots,B,
$$
be bootstrap indices which are drawn independently from the uniform distribution over $[n]=\{1,\ldots,n\}$ \textbf{with replacement}. Then, the $b$-th bootstrap sample is
\begin{equation}\label{S-b}
  S^{(b)}=(\bz_{i_{b,1}},\ldots,\bz_{i_{b,m}}).
\end{equation}

\begin{definition}[Bootstrap method\label{def:type}]
Let $S^{(b)}$ be defined as Eq. \eqref{S-b}. Then, we apply an algorithm to $S^{(b)}$ and get a model $\bw^{(b)}$. The bootstrap method outputs one of the following models for final prediction:
\begin{itemize}
\item Type 1: $h_S^{(1)}=h_{\bar{\bw}}$, where $\bar{\bw}=\frac{1}{B}\sum_{b=1}^{B}\bw^{(b)}$.
\item Type 2: $h_S^{(2)}=\frac{1}{B}\sum_{b=1}^{B}h_{\bw^{(b)}}$.
\item Type 3: $h_S^{(3)}(x)=\mathrm{median}_{1\le b\le B} \{h_{\bw^{(b)}}(x)\}$, $x\in\xcal$.
\end{itemize}
\end{definition}

\begin{definition}[Bootstrap SGD\label{def:sgd}]
For the $b$-th data set, at the $t$-th iteration, bootstrap SGD first selects $j_{b,t}$ from the uniform distribution over $[m]$ and then update $\bw_t^{(b)}$ as follows
\[
\bw_{t+1}^{(b)}=\bw_{t}^{(b)}-\eta_t\nabla f(\bw_{t}^{(b)};\bz_{i_{b,j_{b,t}}}),
\]
where $(\eta_t)_{t\in\N}$ is a sequence of positive step sizes.
After $T$ iterations, bootstrap SGD outputs $A(S)=\frac{1}{B}\sum_{b=1}^{B}\bw_{T+1}^{(b)}$.
If not otherwised mentioned, we use $\bw_{1}^{(b)}=0$ as the starting value.
\end{definition}

\section{Algorithmic Stability of Bootstrap SGD\label{sec:stab}}

\subsection{Stability Analysis for Bootstrap SGD of Type 1}

Let $\widetilde{S}=(\tilde{\bz}_1,\ldots,\tilde{\bz}_n)$ be a neighbouring dataset of $S$, i.e., $S$ and $\widetilde{S}$ only differ by a single example.
We will assume in this paper, if not mentioned otherwise, that the learning algorithm is
permutation invariant. Hence without loss of generality, we can assume that $S$ and $\widetilde{S}$ differ by the last example, i.e.,
\[
\bz_i=\tilde{\bz}_i,\quad\text{if }i<n.
\]
We assume $\bz_1,\ldots,\bz_n,\tilde{\bz}_n$ are independently drawn from the same distribution. The proofs of results in this subsection are given in Section~\ref{sec:proof-type1}.

\begin{lemma}\label{lem:sgd}
  Let $(\bw_t^{(b)})$ and $(\tilde{\bw}_t^{(b)})$ be produced by bootstrap SGD based on $S$ and $\widetilde{S}$, respectively. Assume the map $\bw\mapsto f(\bw;\bz)$ is convex and $L$-smooth. If $\eta_t\leq 1/L$, then
  \[
  \big\|\bw_{T+1}^{(b)}-\tilde{\bw}_{T+1}^{(b)}\big\|_2\leq \sum_{t=1}^{T}\eta_t\delta_{t,b}\ibb[i_{b,j_{b,t}}=n],
  \]
  where we introduce
  \[
  \delta_{t,b}=\|\nabla f(\bw_{t}^{(b)};\bz_{n})-\nabla f(\bw_{t}^{(b)};\tilde{\bz}_{n})\|_2
  \]
  and $\ibb[i_{b,j_{b,t}}=n]$ is the indicator function, taking value $1$ if $i_{b,j_{b,t}}=n$, and $0$ otherwise.
\end{lemma}
As we will show in Lemma \ref{lem:exp-indicator}, $\ebb_A[\ibb[i_{b,j_{b,t}}=n]]=\frac{1}{n}$. Therefore, we can apply Lemma \ref{lem:sgd} to get the following theorem on the stability of bootstrap SGD.

\begin{theorem}[Stability bounds\label{thm:sgd}]
  Let assumptions in Lemma \ref{lem:sgd} hold. Then
  \[
  \ebb_A\big[\|A(S)-A(\widetilde{S})\|_2\big]\leq \sum_{t=1}^{T}\eta_t\ebb_A[\delta_{t,1}\ibb[i_{1,j_{1,t}}=n]]
  \]
  and
  \begin{multline*}
    \ebb_A\big[\|A(S)-A(\widetilde{S})\|_2^2\big]
     \leq \frac{2}{mB}\sum_{r=1}^{m}\frac{rC_m^r(n-1)^{m-r}}{n^m}\cdot\ebb_A\Big[\sum_{t=1}^{T}\eta_t^2\delta_{t,1}^2\big|\sum_{k=1}^m\ibb[i_{1,k}=n]=r\Big]+\\
  \frac{2}{m^2B}\sum_{r=0}^{m}\frac{r^2C_m^r(n-1)^{m-r}}{n^m}\cdot\ebb_A\Big[\Big(\sum_{t=1}^{T}\eta_t\delta_{t,1}\Big)^2\big|\sum_{k=1}^m\ibb[i_{1,k}=n]=r\Big]+ (1-\frac{1}{B})\Big(\ebb_A\Big[\sum_{t=1}^{T}\eta_t\delta_{t,1}\ibb[i_{1,j_{1,t}}=n]\Big]\Big)^2.
  \end{multline*}
\end{theorem}

\begin{remark}
  The hypothesis stability of bootstrap methods has been studied in \citep{elisseeff2005stability}. Specifically, assume the base machine has hypothesis stability $\gamma_k$ for learning with a dataset of size $k$ and $m=n$. Then, the hypothesis stability parameter $\beta$ of the bootstrap method satisfies \citep{elisseeff2005stability}
  \[
  \beta\leq G\sum_{k=1}^{n}\frac{k\gamma_k}{n}\mathrm{Pr}\big\{d(I)=k\big\},
  \]
  where $I=(i_1,\ldots,i_n)$ and $d(I)$ denotes the number of distinct sampled points in one bootstrap sample. Here $i_j,j\in[n]$ follows the uniform distribution over $[n]$. However, the analysis there treats multiple copies of a training point as one point, which is not the case in practice. It is also indicated there how to extend their results to the case where multiple copies of a point are treated is an open question~\citep{elisseeff2005stability}. As a comparison, our analysis does not require this simplifying assumption and therefore provides an affirmative solution to their open question. It should be mentioned that the analysis in \citep{elisseeff2005stability} considers bootstrap with a general base machine. As a comparison, this paper considers bootstrap with the base algorithm being SGD.
\end{remark}
\begin{corollary}\label{cor:sgd}
  Let assumptions in Lemma \ref{lem:sgd} hold. Assume that the map $\bw\mapsto f(\bw;\bz)$ is $G$-Lipschitz continuous. Then
  \[
  \ebb_A\big[\|A(S)-A(\widetilde{S})\|_2\big]\leq\frac{2G^2}{n}\sum_{t=1}^{T}\eta_t
  \]
  and
    \[
    \ebb_A\big[\|A(S)-A(\widetilde{S})\|_2^2\big]
     \leq \frac{8G^2}{Bn}\sum_{t=1}^{T}\eta_t^2+
     \frac{8G^2}{Bn}\Big(\frac{2}{n}+\frac{1}{m}\Big)\Big(\sum_{t=1}^{T}\eta_t\Big)^2
     + \frac{8G^2(B-1)}{n^2B}\Big(\sum_{t=1}^{T}\eta_t\Big)^2.
  \]
\end{corollary}
\begin{remark}
  Since $m\leq n$, we get
  \[
    \ebb_A\big[\|A(S)-A(\widetilde{S})\|_2^2\big]
     \lesssim \frac{G^2}{Bn}\sum_{t=1}^{T}\eta_t^2+
     G^2\Big(\frac{1}{Bmn}+\frac{1}{n^2}\Big)\Big(\sum_{t=1}^{T}\eta_t\Big)^2,
  \]
  where we denote $A\lesssim B$ if there exists a universal constant $C$ such that $A\leq CB$. We also denote $A\gtrsim B$ if there exists a universal constant $C$ such that $A\geq CB$. We denote $A\asymp B$ if $A\lesssim B$ and $A\gtrsim B$.
\end{remark}

We can combine the above stability bounds and Lemma \ref{lem:stab-gen} to derive generalization bounds in Theorem \ref{thm:gen-sgd}. We omit the proof for brevity.
\begin{theorem}[Generalization bounds\label{thm:gen-sgd}]
Let $A$ be the bootstrap SGD with $T$ iterations. Suppose for any $z$, the map $\bw\mapsto f(\bw;z)$ is convex, $G$-Lipschtiz continuous and $L$-smooth. Then
\[
  \ebb[F(A(S))-F_S(A(S))] \leq \frac{2G^2}{n}\sum_{t=1}^{T}\eta_t
\]
and
\begin{multline}\label{gen-sgd-b}
  \ebb[F(A(S))-F_S(A(S))] \lesssim \frac{LG^2}{Bn}\sum_{t=1}^{T}\eta_t^2+
     \Big(\frac{LG^2}{n^2}+\frac{LG^2}{Bmn}\Big)\Big(\sum_{t=1}^{T}\eta_t\Big)^2+\\
     G\big(L \ebb[F_S(A(S))]\big)^{\frac{1}{2}}\Big(\frac{1}{Bn}\sum_{t=1}^{T}\eta_t^2+
     \Big(\frac{1}{n^2}+\frac{1}{Bmn}\Big)\Big(\sum_{t=1}^{T}\eta_t\Big)^2\Big)^{\frac{1}{2}}.
\end{multline}
\end{theorem}
\begin{remark}
  If we choose $\eta_t\asymp 1/\sqrt{T}$ and $T\asymp n$, then Eq.~\eqref{gen-sgd-b} further implies
  \[
  \ebb[F(A(S))-F_S(A(S))] \lesssim
     \frac{LG^2}{n}+\frac{LG^2}{Bm}+
     G\big(L \ebb[F_S(A(S))]\big)^{\frac{1}{2}}\Big(
     \frac{1}{n}+\frac{1}{Bm}\Big)^{\frac{1}{2}},
  \]
  which is a decreasing function of $Bm$. This shows that the bootstrap helps improve the generalization behavior of SGD, which is reasonable since sampling and averaging intuitively improves the stability of a learning algorithm. The above generalization bound also involves the empirical risk and therefore can benefit from a small training error. That is, if $\ebb[F_S(A(S))]$ is very small, then the generalization bound can improve to the order of $\frac{LG^2}{n}+\frac{LG^2}{Bm}$.
\end{remark}

\subsection{Stability Analysis for Bootstrap SGD of Type 2}
In this section, we consider stability and generalization bounds for bootstrap SGD of Type 2. The following theorem gives the stability bounds. The proofs of results in this subsection are given in Section~\ref{sec:proof-type2}.
\begin{theorem}[Stability bound\label{thm:stab-type2}]
Let $(\bw_t^{(b)})$ and $(\tilde{\bw}_t^{(b)})$ be produced by bootstrap SGD based on $S$ and $\widetilde{S}$, respectively. Denote by $h^{(2)}_S$ the bootstrap output of Type 2. Assume the map $\bw\mapsto f(\bw;\bz)$ is convex and $L$-smooth. Assume $h_{\bw}$ is $G_h$-Lipschitz continuous in the sense that
\begin{equation}\label{lipschitz-h}
\big|h_{\bw}(x)-h_{\bw'}(x)\big|\leq G_h\|\bw-\bw'\|_2,\quad\forall \bx\in\xcal.
\end{equation}
If $\eta_t\leq 1/L$, then
\begin{equation}\label{stab-type2-a}
\ebb_A\big[\sup_{\bx}\big|h^{(2)}_S(\bx)-h^{(2)}_{\widetilde{S}}(\bx)\big|\big] \leq
G_h\sum_{t=1}^{T}\eta_t\ebb_A[\delta_{t,1}\ibb[i_{1,j_{1,t}}=n]]
\end{equation}
and
\begin{multline*}
  \ebb_A\big[\sup_{\bx}\big(h^{(2)}_S(\bx)-h^{(2)}_{\widetilde{S}}(\bx)\big)^2\big] \leq
  \frac{2G_h^2}{mB}\sum_{r=1}^{m}\frac{rC_m^r(n-1)^{m-r}}{n^m}\cdot\ebb_A\Big[\sum_{t=1}^{T}\eta_t^2\delta_{t,1}^2\big|\sum_{k=1}^m\ibb[i_{1,k}=n]=r\Big]+\\
  \frac{2G_h^2}{m^2B}\sum_{r=0}^{m}\frac{r^2C_m^r(n-1)^{m-r}}{n^m}\cdot\ebb_A\Big[\Big(\sum_{t=1}^{T}\eta_t\delta_{t,1}\Big)^2\big|\sum_{k=1}^m\ibb[i_{1,k}=n]=r\Big]+ G_h^2(1-\frac{1}{B})\Big(\ebb_A\Big[\sum_{t=1}^{T}\eta_t\delta_{t,1}\ibb[i_{1,j_{1,t}}=n]\Big]\Big)^2.
\end{multline*}
\end{theorem}

Theorem~\ref{thm:stab-type2} requires $h$ to be Lipschitz continuous. Below we give two examples of such $h$.
\begin{example}[Linear models]
    Consider a linear model $h_{\bw}(x)=\bw^\top x$. Then, it is clear that
    \[
    |h_{\bw}(x)-h_{\bw'}(x)|=\big|(\bw-\bw')^\top x\big|\leq \|\bw-\bw'\|_2\sup_{x\in\xcal}\|x\|_2.
    \]
    Therefore, Eq.~\eqref{lipschitz-h} holds with $G_h=\sup_{x\in\xcal}\|x\|_2$. It is clear that Eq.~\eqref{lipschitz-h} also holds for kernel models with bounded kernels.
\end{example}

\begin{example}[Shallow neural networks]
    Consider a shallow neural network (SNN) with $N$ nodes in the hidden layer
    \[
    h_{\bW}(x)=\sum_{j=1}^N\alpha_j\sigma(\bw_j^\top x),
    \]
    where $\alpha_j\in\{-1/\sqrt{N},1/\sqrt{N}\}$ are fixed, $\sigma:\rbb\mapsto\rbb$ is an activation function and $\bW=(\bw_1,\ldots,\bw_N)$ are trainable weights with $\bw_j\in\rbb^d$. If $\sigma$ is $G_\sigma$-Lipschitz continuous, then, for any $\bW,\bW'$, we have
    \begin{align*}
      \big|h_{\bW}(x)-h_{\bW'}(x)\big| & = \Big|\sum_{j=1}^N\alpha_j\big(\sigma(x^\top\bw_j)-\sigma(x^\top\bw_j')\big)\Big|
       \leq \frac{1}{\sqrt{N}}\sum_{j=1}^{N}\big|\sigma(x^\top\bw_j)-\sigma(x^\top\bw_j')\big|\\
       & \leq \frac{G_\sigma}{\sqrt{N}}\sum_{j=1}^{N}\big|x^\top\big(\bw_j-\bw_j'\big)\big|
       \leq \frac{G_\sigma\sup_{x}\|x\|_2}{\sqrt{N}}\sum_{j=1}^{N}\big\|\bw_j-\bw_j'\big\|_2\\
       & \leq \frac{G_\sigma\sup_{x}\|x\|_2}{\sqrt{N}}\sqrt{N}\Big(\sum_{j=1}^{N}\big\|\bw_j-\bw_j'\big\|_2^2\Big)^{\frac{1}{2}}\\
       & = G_\sigma\sup_{x}\|x\|_2\|\bW-\bW'\|_2,
    \end{align*}
    where $\|\cdot\|_2$ denotes the Frobenius norm of a matrix.
    Therefore, Eq.~\eqref{lipschitz-h} holds for SNNs with $G_h=G_\sigma\sup_{x}\|x\|_2$. Although a SNN is not convex, it was shown that it is $\mu$-weakly convex with the weak convexity parameter $\mu$ (the smallest eigenvalue of a Hessian matrix) being of the order of $1/\sqrt{N}$~\citep{richards2021stability,taheri2024generalization}. Then, our stability analysis still applies for sufficiently large $N$~\citep{lei2022stability,richards2021stability}.
\end{example}

As a corollary, we present the following generalization bounds. It shows that bootstrap SGD of Type 2 enjoys similar stability and generalization bounds of bootstrap SGD of Type 1.
\begin{corollary}[Generalization bound\label{cor:gen-type2}]
Let assumptions in Theorem \ref{thm:stab-type2} hold. Assume $a\mapsto\ell(a,y)$ is $G_\ell$-Lipschitz continuous for any $y$ and $\bw\mapsto f(\bw;z)$ is $G$-Lipschitz continuous for any $z$. Then
\begin{equation}\label{gen-type2-a}
\ebb_A[\ell(h_S^{(2)}(x),y)]-\frac{1}{n}\sum_{i=1}^{n}\ebb_A[\ell(h_S^{(2)}(x_i),y_i)]\leq\frac{2GG_hG_\ell}{n}\sum_{t=1}^{T}\eta_t.
\end{equation}
Furthermore, if for any $y$, the function $a\mapsto\ell(a,y)$ is $L_\ell$-smooth, then
\begin{multline}\label{gen-type2-b}
  \ebb_A[\ell(h_S^{(2)}(x),y)]-\frac{1}{n}\sum_{i=1}^{n}\ebb_A[\ell(h_S^{(2)}(x_i),y_i)] \lesssim G_\ell GG_h\bigg(\frac{1}{Bn}\sum_{t=1}^{T}\eta_t^2+
     \Big(\frac{1}{Bmn}+\frac{1}{n^2}\Big)\Big(\sum_{t=1}^{T}\eta_t\Big)^2\bigg)^{\frac{1}{2}}\\
     +\frac{G^2G_h^2L_\ell}{Bn}\sum_{t=1}^{T}\eta_t^2+
     G^2G_h^2L_\ell\Big(\frac{1}{Bmn}+\frac{1}{n^2}\Big)\Big(\sum_{t=1}^{T}\eta_t\Big)^2.
\end{multline}
\end{corollary}
\begin{remark}
  If we choose $\eta_t\asymp 1/\sqrt{T}$ and $T\asymp n$, then Eq.~\eqref{gen-type2-a} implies that
  \[
  \ebb_A[\ell(h_S^{(2)}(x),y)]-\frac{1}{n}\sum_{i=1}^{n}\ebb_A[\ell(h_S^{(2)}(x_i),y_i)]\lesssim\frac{GG_hG_\ell}{\sqrt{n}},
  \]
  which, however, does not show the effect of the bootstrap. For the same step size and iteration number, Eq.~\eqref{gen-type2-b} shows that
  \[
  \ebb_A[\ell(h_S^{(2)}(x),y)]-\frac{1}{n}\sum_{i=1}^{n}\ebb_A[\ell(h_S^{(2)}(x_i),y_i)] \lesssim G_\ell GG_h
     \Big(\frac{1}{Bm}+\frac{1}{n}\Big)^{\frac{1}{2}}+
     G^2G_h^2L_\ell\Big(\frac{1}{Bm}+\frac{1}{n}\Big).
  \]
  This shows that increasing $B$ would improve the generalization behavior of the output model.
\end{remark}

\section{Distribution-free Confidence and Tolerance Intervals for Bootstrap SGD of Type 3\label{sec:confidence}}

It is well-known that distribution-free confidence intervals for quantiles  and distribution-free tolerance intervals can be constructed by certain
intervals, where the endpoints are defined by order statistics, see
\citet[Chap.~7]{DavidNagaraja2003}. We will use this result for the
special case of the median, although the generalization to other quantiles
is possible in the same manner.

Table \ref{tab:confidenceintervals} lists some values of $B$, the
corresponding pair of order statistics determining the confidence
interval, the lower bound of the actual confidence level which is
$0.5^{B}\,\sum_{j=r}^{s} ( \,_{j}^{B})$, and the finite sample
breakdown point (see Definition \ref{breakdownpoint})
$\varepsilon_B^*=\min\{r-1, B-s \}/B$ of the confidence interval.

\begin{table}[H]
\begin{center}
\begin{tabular}[t]{|r|r|r|r|c|c|}
\hline
$1-\alpha$ & $B$   &  $r$ & $s$   & lower bound of   & finite sample  \\
           &   &   &       & confidence level & breakdown point \\
\hline
0.90  & 8 & 2 & 7 & 0.930 & 0.125 \\
 & 10 & 2 & 9 & 0.979 & 0.100 \\
 & 18 & 6 & 13 & 0.904 & 0.278  \\
 & 30 & 11 & 20 & 0.901 & 0.333 \\
 & 53 & 21 & 33 & 0.902 & 0.377 \\
 & 104 & 44 & 61 & 0.905 & 0.413 \\
\hline
0.95 & 9 & 2 & 8 & 0.961 & 0.111 \\
 & 10 & 2 & 9 & 0.979 & 0.100  \\
 & 17 & 5 & 13 & 0.951 & 0.235 \\
 & 37 & 13 & 25 & 0.953 & 0.324 \\
 & 51 & 19 & 33 & 0.951 & 0.353 \\
 & 101 & 41 & 61 & 0.954 & 0.396 \\
\hline
0.99 & 10 & 1 & 10 & 0.998 & 0.000  \\
 & 12 & 2 & 11 & 0.994 & 0.083 \\
 & 26 & 7 & 20 & 0.991 & 0.231 \\
 & 39 & 12 & 28 & 0.991 & 0.282 \\
 & 49 & 16 & 34 & 0.991 & 0.306 \\
 & 101 & 38 & 64 & 0.991 & 0.366 \\
\hline
\end{tabular}
\caption{Selected pairs $(r,s)$ of order statistics for non-parametric confidence intervals at the
     $(1-\alpha)$-level for the median. \label{tab:confidenceintervals}}
\end{center}
\end{table}

The finite sample breakdown point proposed by \citet{DonohoHuber1983} measures the
worst case behaviour of a statistical estimator. We use the replacement version of this breakdown point, see \citet[p.98]{HampelRonchettiRousseeuwStahel1986}.
 If two estimators are compared w.r.t. the finite sample breakdown point, the one with the higher value  is more robust than the other one.
The influence function proposed by \citet{Hampel1968,Hampel1974}
measures the impact on the estimation due to an infinitesimal small contamination of the distribution $P$ in the direction of a Dirac-distribution.

\begin{definition}[Finite-sample breakdown point]\label{breakdownpoint}
Let $S_n=(z_{1},\ldots, z_{n})$ be a data set with values in
$\zcal$.
The finite-sample breakdown point of learning method $A$ for the data set $S_{n}$ is defined by
\begin{equation}
\varepsilon_n^*(A(S_{n})) = \max \left\{
                              \frac{m}{n} \,; \mathrm{Bias}(m; A(S_{n})) \mathrm{ ~~is~finite}
                         \right \} \, ,
\end{equation}
where
\begin{equation}
 \mathrm{Bias}(m; A(S_{n})) = \sup_{\mathcal{S}_n^{\prime}} \|
        A(S_{n}^{\prime}) - A(S_{n}) \|_2  \label{sec2:bias}
\end{equation}
and the supremum is over all possible samples $S_n^{\prime}$ that can be obtained
by replacing \emph{any} $m$ of the original data points by \emph{arbitrary} values in $\zcal$.
\end{definition}

Table \ref{tab:toleranceintervals} lists pairs $(r,s)$ of indices for order statistics
$(X_{r},X_{s})$ which yield distribution-free tolerance intervals
$[X_{(r)},X_{(s)}]$ such that at least a proportion of
 $\gamma$ of the population is covered with probability $\beta$.
 Here, $s=B-r+1$. In the table the guaranteed lower bounds of the probability and the finite sample breakdown points are given, too. 

 \begin{table}[H]
 \caption{The pairs of order statistics
$X_{(r)}$ and $X_{(s)}$ from $X_{1},\ldots,X_{B}$  yield distribution-free tolerance intervals
$[X_{(r)},X_{(s)}]$, i.e.
$\mathrm{Pr}(\int_{[X_{(r)},X_{(s)}]} \,dP \ge \gamma)\ge \beta$. The final sample breakdown points are $\min\{r-1, B-s+1\}/B$.   \label{tab:toleranceintervals}}
  \begin{center}
 \begin{tabular}[t]{|l|l|r|r|r|r|r|}
 \hline \hline
 $\gamma$ & $\beta$ & $B$ & $r$ & $s$ & prob $\ge$ & finite sample  \\
          &         &     &     &     &            & breakdown point \\
\hline
   0.9  & 0.9   & 38 &  1 & 38 & 0.9047 & 0 \\
   0.9  & 0.95  & 46 &  1 & 46 & 0.9519 & 0 \\
   0.9  & 0.99  & 64 &  1 & 64 & 0.9904 & 0 \\
   0.9  & 0.999 & 89 &  1 & 89 & 0.9990 & 0 \\
\hline
    0.95  & 0.90 &  77 &   1 & 77 & 0.9026 & 0\\
    0.95  & 0.95 &  93 &   1 & 93 & 0.9500 & 0\\
    0.95  & 0.99 & 130 &   1 & 130 & 0.9900 & 0\\
    0.95 & 0.999 & 181 &   1 & 181 & 0.9990 & 0\\
\hline
     0.99  &  0.90  & 388 &   1 & 388 & 0.9003 & 0\\
     0.99  &  0.95  & 473 &   1 & 473 & 0.9502 & 0\\
     0.99  &  0.99  & 662 &   1 & 662 & 0.9900 & 0\\
     0.99 &  0.999  & 920 &   1 & 920 & 0.9990 & 0\\
 \hline \hline
   0.9  & 0.9   & 65  &  2 & 64  & 0.9004 & 0.031   \\
   0.9  & 0.95  & 76  &  2 & 75  & 0.9530 & 0.026  \\
   0.9  & 0.99  & 97  &  2 & 96  & 0.9901 & 0.021  \\
   0.9  & 0.999 & 126 &  2 & 125 & 0.9990 & 0.016  \\
\hline
    0.95  & 0.90 & 132 &   2 & 131 & 0.9007 & 0.015  \\
    0.95  & 0.95 & 153 &   2 & 152 & 0.9505 & 0.013 \\
    0.95  & 0.99 & 198 &   2 & 197 & 0.9902 & 0.010  \\
    0.95 & 0.999 & 257 &   2 & 256 & 0.9990 & 0.008  \\
\hline
     0.99  &  0.90  &  667 &   2 &  666 & 0.9004 & 0.003\\
     0.99  &  0.95  &  773 &   2 &  772 & 0.9500 & 0.003  \\
     0.99  &  0.99  & 1001 &   2 & 1000 & 0.9900 & 0.002 \\
     0.99 &  0.999  & 1302 &   2 & 1301 & 0.9990 & 0.002 \\
\hline \hline
   0.9  & 0.9   & 164 &  6 & 159 & 0.9037  & 0.037\\
   0.9  & 0.95  & 179 &   6 & 174 & 0.9514 & 0.034\\
   0.9  & 0.99  & 210 &   6 & 205 & 0.9902 & 0.029\\
   0.9  & 0.999 & 249 &   6 & 244 & 0.9990 & 0.024 \\
\hline
    0.95  & 0.90 &  330 &   6 & 325 & 0.9017 & 0.018\\
    0.95  & 0.95 &  361 &   6 & 356 & 0.9505 & 0.017 \\
    0.95  & 0.99 &  425 &   6 & 420 & 0.9901 & 0.014 \\
    0.95 & 0.999 &  505 &   6 & 500 & 0.9990 & 0.012 \\
\hline
     0.99  &  0.90  & 1658 &    6 & 1653 & 0.9004 & 0.004 \\
     0.99  &  0.95  & 1818 &    6 & 1813 & 0.9501 & 0.003 \\
     0.99  &  0.99  & 2144 &    6 & 2139 & 0.9900 & 0.003 \\
     0.99 &  0.999 & 2552 &    6 & 2547 & 0.9990 & 0.002 \\
 \hline \hline
\end{tabular}
\end{center}
\end{table}
\ACe


It is possible to construct distribution-free prediction intervals based on order statistics in a very similar manner. We refer for details to \citet[Chap.~7.3]{DavidNagaraja2003}.

\section{Numerical Example for the 3 Types\label{sec:example}}
In this section, we present experimental results to show the behavior of bootstrap  SGD  with different types.

For illustration purposes only, let us demonstrate how the bootstrap methods behave in a one dimensional toy example.
The data set is generated as follows. The sample size is $n=1000$.
The input values $x_{i}$ are generated in an i.i.d. manner from a uniform distribution on the interval $[0, 33]$.
The output values $y_{i}$ are generated in the usual signal plus noise manner, i.e. as realisations of stochastically independent random variables $Y_{i}$ such that $y_{i}$ is the observed value $Y_{i}|(X_{i}=x_{i})$ having distribution
$f(x_{i}) + \varepsilon_{i}(x)$. To consider several cases of the ``true'' function
$f$ and of the error distribution $\varepsilon$ simultaneously, we have chosen
\begin{eqnarray*}
 f(x)  & = &
\begin{cases}
        0.7 x       & \textrm{~~if~} x  \in [0,3] \\
        10+x+\frac{1}{100}\sin(10x) x^4        & \textrm{~~if~} x  \in (3,6] \\
        5   & \textrm{~~if~} x  \in (6,30] \\
        -20-0.4(x-27)^2       & \textrm{~~if~} x  \in (30,33] \\
\end{cases}
\end{eqnarray*}
and
\begin{eqnarray*}
 \varepsilon(x) & \sim &
\begin{cases}
         \textrm{Unif}(-1,+1)      & \textrm{~~if~} x  \in [0,3] \\
        \textrm{Exp}(0.5)-\textrm{median}(\textrm{Exp}(0.5))
                  & \textrm{~~if~} x  \in (3,6] \\
         \textrm{Cauchy}(0,1) & \textrm{~~if~} x  \in (6,30] \\
         N(0,4)      & \textrm{~~if~} x  \in (30,33]. \\
\end{cases}
\end{eqnarray*}
Obviously, the function $f$ is non-continuous and therefore not an element of the RKHS of any Gaussian RBF kernel. Furthermore the function is constant, linear, polynomial of order $2$, and more complex in different intervals.
We used the uniform distribution on the interval $(-1,1)$ as an example
of a symmetrical distribution with a compact support, the Gaussian distribution as an example of a symmetrical distribution with thin tails, the Cauchy distribution with median $0$ and scale parameter $1$ as an extreme example of a symmetrical distribution with heavy tails, and the exponential distribution shifted by its median as an example of a skewed distribution with median equal to zero.

We used the classical Gaussian radial basis function kernel with $K(x,x')=\exp\big(-(x-x')^2/(2\sigma^2)\big)$  and $\sigma=1/\sqrt{20}$.  We used simple random sampling with replacement to generate $B=101$ bootstrap samples. For the $b$-th bootstrap sample, we applied SGD with $\eta_t=10/\sqrt{t}$ and $400$ passes over the data to generate a model $\bw_{T+1}^{(b)}$, which was then used to compute the predicted output $h_{\bw_T}^{(b)}$. We then applied the bootstrap methods in Definition~\ref{def:type} to compute the final prediction. Since kernel models are linear in the corresponding reproducing kernel Hilbert spaces, Type 1 and Type 2 are the same in our setup.

In Figure~\ref{fig:type2}, we illustrate the behavior of bootstrap methods of Type 2. We also plot the  pointwise  $0.95$-confidence interval predictions on the dataset to show the variation of the predictions.

In Figure~\ref{fig:type3}, we illustrate the behavior of bootstrap methods of Type 3. As we used $B=101$, we used $r=41$ and $s=61$ to construct  pointwise  distribution free confidence intervals with $1-\alpha=0.95$, see Table \ref{tab:confidenceintervals}.

It is clear that for this toy example both methods yield very similar results
from an applied point of view and that both bootstrap methods are robust in producing reliable predictions.
We like to emphasize that we had to zoom in on the $y$-axis, as we even allowed
Cauchy-distributed errors. The minimum and maximum of the $y$-values in the data set are $-261.69$ and $1726.99$, respectively. Of course, all computations and all bootstrap replications were done
using all data points, even these extreme values. 

\begin{figure}
  \centering
  \includegraphics[width=0.7\textwidth]{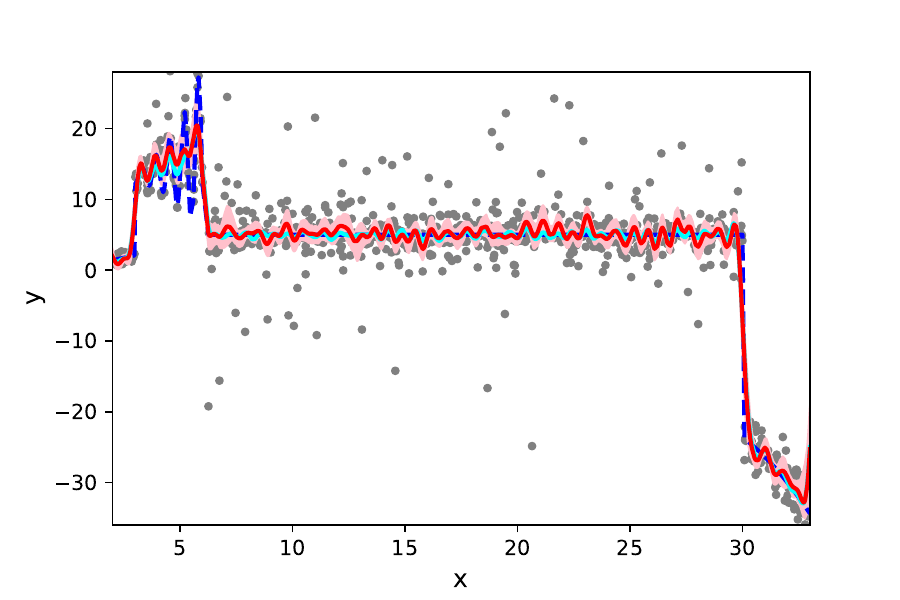}
  \caption{Behavior of Bootstrap Method of Type 2, $n=1000$ and $B=101$. The shaded scattered points show the training examples, the blue plot shows the true function, the cyan line shows the SGD approximation of the true data, the pink area shows the  pointwise $0.95$-confidence intervals, and the red line shows the average of the bootstrap approximations.}\label{fig:type2}
\end{figure}

\begin{figure}
  \centering
  \includegraphics[width=0.7\textwidth]{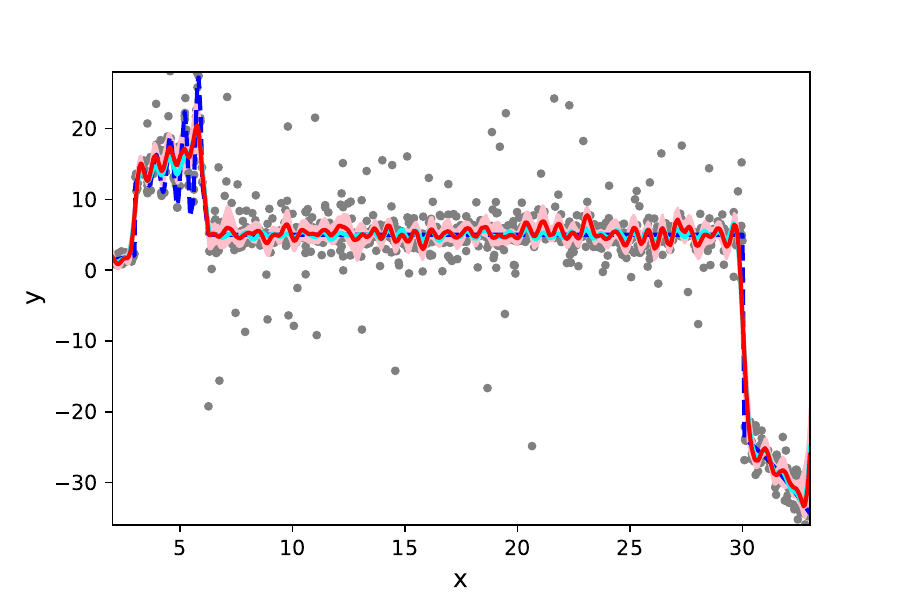}
  \caption{Behavior of Bootstrap Method of Type 3, $n=1000$ and $B=101$. The shaded scattered points show the training examples, the blue plot shows the true function, the cyan line shows the SGD approximation of the true data, the pink area shows the  pointwise $0.95$-confidence intervals for the median, and the red line shows the median of the bootstrap approximations.}\label{fig:type3}
\end{figure}

\section{Discussion\label{sec:discussion}}
In this paper three methods to use the empirical bootstrap approach for
SGD to minimize the empirical risk over a separable Hilbert space were investigated from the view point of
algorithmic stability and statistical robustness.
In Type 1 and Type 2 one simply computes the average from the bootstrap
approximations which yield
$h_S^{(1)}=h_{\bar{\bw}}$, where
$\bar{\bw}=\frac{1}{B}\sum_{b=1}^{B}\bw_{T+1}^{(b)}$, and
$h_S^{(2)}=\frac{1}{B}\sum_{b=1}^{B} h_{\bw_{T+1}^{(b)}}$, respectively.
These two types of bootstrap SGD have the property
that the estimated functions $h_S^{(1)}$ and
$h_S^{(2)}$ are both elements of the Hilbert space.
In Type 3, one computes the pointwise median of the bootstrap approximations
which yields
$h_S^{(3)}(x)=\mathrm{median}_{1\le b\le B} \{h_{\bw_{T+1}^{(b)}}(x)\}$, $x\in\xcal$. Our results show that these bootstrap SGD methods
have some desirable algorithmic stability and robustness properties if
the loss function has some easy to check properties, e.g. convexity, smoothness and Lipschitz continuity. In our toy example we demonstrated that it can even handle Cauchy distributed errors.\\[1em]

We presented generalization analysis for bootstrap SGD of Type 1 and Type 2 based on algorithmic stability. We considered both $\ell_1$- and $\ell_2$-argument stability. 
For SGD of Type 1, the $\ell_1$-argument stability does not show the effect of bootstrap.
The underlying reason is that the $\ell_1$-argument stability fails to incorporate the second-order information such as the variance.
As a comparison, the $\ell_2$-argument stability analysis yields stability bounds of the order of
$$
  \Big(\frac{1}{Bn}\sum_{t=1}^T\eta_t^2\Big)^{1/2}+\Big(\frac{1}{Bmn}+\frac{1}{n^2}\Big)^{1/2} \sum_{t=1}^T\eta_t~.
$$
This shows that the increasing the number of bootstrap samples improves the stability and generalization behavior of the output models. Our $\ell_2$-argument stability analysis exploits the variance reduction property by  bootstrapping. The pioneering work~\citep{elisseeff2005stability} studied the hypothesis stability of bootstrap methods, which, however treated multiple copies of a training point as one point. It was posed as an open question there on how to handle multiple copies as different training points. Our stability analysis provides an affirmative solution to this open question. For bootstrap SGD of Type 2, we impose an additional assumption on the Lipschitz continuity of the prediction function, which holds for many models such as linear models and shallow neural networks. To our knowledge, this is the first stability analysis for bootstrap SGD of Type 2.

Bootstrap SGD of Type 3 allows the construction of pointwise \emph{distribution-free}
confidence and \emph{distribution-free} tolerance intervals using a well-known fact
from order statistics. This type of bootstrap SGD has the property that the estimated median regression curve is not necessarily an element of the Hilbert space of functions which is used as the hypothesis space. Whether this property
may be considered as an advantage or as a disadvantage probably depends on the
specific application. E.g., if the Hilbert space only contains very smooth functions,
which is true if a classical Gaussian RBF kernel is used, but the Bayes function
is non-continuous, this property may be considered as an advantage.

It seems possible to enhance the robustness and stability properties of
bootstrap SGD even more. Although the random variables in each bootstrap sample are drawn in an independent and identically distributed manner from the empirical distribution computed from the training data set, it can happen that a few bootstrap samples are ``untypical'' in the sense
that the empirical distribution of the bootstrap sample differs a lot from the
empirical distribution of the training data set.
This can happen e.g. if just by chance some data points do occur very often in a specific bootstrap sample with index $b^\prime$.
In this case, it can happen that for this specific bootstrap sample the quantities $\bw_{T+1}^{(b^\prime)}$ (or $h_{\bw_{T+1}^{(b^\prime)}}$) differ a lot from the \emph{majority} of the other values of $\bw_{T+1}^{(b)}$ (or $h_{\bw_{T+1}^{(b)}}$), $b\ne b^\prime$.
 This phenomenon can have some undesired consequences in particular to
data sets with a small to moderate sample size $n$.
Then an empirical bootstrap estimate can be strongly influenced by this outlying value, if one uses the classical average of the values $\bw_{T+1}^{(b)}$ in Type 1 or of the values $h_{\bw_{T+1}^{(b)}}$ in Type 2.
For example, imagine that in the toy example given in Section \ref{sec:example}, the extreme $y$-value $1726.99$ occurs in a bootstrap sample several times.
One idea to robustify the empirical
bootstrap SGD in this case is to replace the average by a more robust estimator.
One example is to use a one-step W-estimator, see  \citet[p. 116]{HampelRonchettiRousseeuwStahel1986}.
Let us briefly sketch this approach although it is beyond the scope of this paper.
First, we use SGD to compute $\bw_{T+1}^{(tr)}$ and $h_{\bw_{T+1}}^{(tr)}$
for the \emph{training} data set.
Then a one-step W-estimator for bootstrap SGD of $\bw$ based on the bootstrap samples can be defined by
\begin{equation}\label{onestepWestimator}
  \widehat{\bw}_{OS}=\frac{\sum_{b=1}^{B} \bw_{T+1}^{(b)} \cdot u_1(\bw_{T+1}^{(b)} - \bw_{T+1}^{(tr)})}
  {\sum_{b=1}^{B} u_1(\bw_{T+1}^{(b)} - \bw_{T+1}^{(tr)})},
\end{equation}
where $u_1:\R^d\to (0,1]$ is an appropriately chosen positive weight function which downweights outlying quantities.
In a similar manner one can define a one-step W-estimator for bootstrap SGD of $h$ by
\begin{equation}
  \widehat{h}_{OS}=\frac{\sum_{b=1}^{B} h_{\bw_{T+1}^{(b)}} \cdot
                  u_2(h_{\bw_{T+1}^{(b)}} - h_{\bw_{T+1}^{(tr)}})}
  {\sum_{b=1}^{B} u_2(h_{\bw_{T+1}^{(b)}} - h_{\bw_{T+1}^{(tr)}})}.
\end{equation}
This generalizes the bootstrap SGD estimators considered in this paper,
as is obvious from the special case that the functions $u_1$ and $u_2$ are equal to the constant function $1$.
A statistical analysis of such one-step W-estimators will probably
require a modification of the first inequality in {(\ref{sgd-1})}.
Please note that these one-step W-estimators in {(\ref{onestepWestimator})} are different, but have a similar structure than the multivariate M-estimators investigated in \citet[p. 184, (6.16)]{MaronnaMartinYohai2006}.

\section{Proofs\label{sec:proof}}

\subsection{Proofs for Bootstrap SGD of Type 1\label{sec:proof-type1}}
The following lemma establishes the non-expansiveness of gradient operator $\bw\mapsto \bw-\eta\nabla \ell(\bw)$, which was established by \citet{hardt2016train}.
\begin{lemma}[\citep{hardt2016train}\label{lem:hardt}]
Suppose the function $\ell$ is convex and $L$-smooth. Then for all $\bw_1$, $\bw_2$ $\in \wcal$ and $\eta \leq 1/L$, we have
\[ \|\bw_1 - \eta\nabla \ell(\bw_1) - \bw_2 + \eta\nabla \ell(\bw_2 )\|_2 \leq \|\bw_1 - \bw_2 \|_2. \]
\end{lemma}

\begin{proof}[Proof of Lemma \ref{lem:sgd}]
  According to Definition \ref{def:sgd}, we know that
  \[
  \bw_{t+1}^{(b)}-\tilde{\bw}_{t+1}^{(b)}=\bw_{t}^{(b)}-\tilde{\bw}_{t}^{(b)}-\eta_t\nabla f(\bw_{t}^{(b)};\bz_{i_{b,j_{b,t}}})+\eta_tf(\tilde{\bw}_{t}^{(b)};\tilde{\bz}_{i_{b,j_{b,t}}}).
  \]
  We consider two cases. If $i_{b,j_{b,t}}\neq n$, then $\bz_{i_{b,j_{b,t}}}=\tilde{\bz}_{i_{b,j_{b,t}}}$ and therefore we can apply Lemma \ref{lem:hardt} to show that
  \begin{align*}
  \big\|\bw_{t+1}^{(b)}-\tilde{\bw}_{t+1}^{(b)}\big\|_2&=\big\|\bw_{t}^{(b)}-\tilde{\bw}_{t}^{(b)}-\eta_t\nabla f(\bw_{t}^{(b)};\bz_{i_{b,j_{b,t}}})+\eta_tf(\tilde{\bw}_{t}^{(b)};{\bz}_{i_{b,j_{b,t}}})\big\|_2\\
  & \leq \big\|\bw_{t}^{(b)}-\tilde{\bw}_{t}^{(b)}\big\|_2.
  \end{align*}
  Otherwise, we know
  \begin{align*}
  \big\|\bw_{t+1}^{(b)}-\tilde{\bw}_{t+1}^{(b)}\big\|_2&=\big\|\bw_{t}^{(b)}-\tilde{\bw}_{t}^{(b)}-\eta_t\nabla f(\bw_{t}^{(b)};\bz_{n})+\eta_tf(\tilde{\bw}_{t}^{(b)};\tilde{\bz}_{n})\big\|_2\\
  & \leq \big\|\bw_{t}^{(b)}-\tilde{\bw}_{t}^{(b)}\big\|_2+\eta_t\|\nabla f(\bw_{t}^{(b)};\bz_{n})-\nabla f(\bw_{t}^{(b)};\tilde{\bz}_{n})\|_2.
  \end{align*}
  We combine the above two cases to derive that
  \[
  \big\|\bw_{t+1}^{(b)}-\tilde{\bw}_{t+1}^{(b)}\big\|_2\leq \big\|\bw_{t}^{(b)}-\tilde{\bw}_{t}^{(b)}\big\|_2+\eta_t\delta_{t,b}\ibb[i_{b,j_{b,t}}=n].
  \]
 We apply the above inequality recursively and derive the stated bound (note $\bw_{1}^{(b)}=\tilde{\bw}_{1}^{(b)}$). The proof is completed.
\end{proof}
\begin{lemma}\label{lem:comb}
  Let $i_1,\ldots,i_m$ follow from the uniform distribution over $[n]$. Then for any $r\in[m]$ we know
  \[
  \pr\Big\{\sum_{t=1}^{m}\ibb[i_t=n]=r\Big\}=\frac{C_m^r(n-1)^{m-r}}{n^m}.
  \]
\end{lemma}
\begin{proof}
  Since $i_t\in[n]$, the vector $(i_1,\ldots,i_m)$ can take $n^m$ different possible values. Furthermore, suppose $\sum_{t=1}^{m}\ibb[i_t=n]=r$. We can first consider all possible $\{j_1,\ldots,j_r\}$ such that $i_{j_k}=n$ for all $k\in[r]$, and there are $C_m^r$ choices. For each $k\not\in\{j_1,\ldots,j_r\}$, $i_k$ takes $n-1$ possible values. Therefore, the number of outcomes with $\sum_{t=1}^{m}\ibb[i_t=n]=r$ is $C_m^r(n-1)^{m-r}$.
  The stated inequality then follows directly. %
  The proof is completed.
\end{proof}
\begin{lemma}\label{lem:exp-indicator}
Let $i_{b,j_{b,t}}$ be defined in Definition \ref{def:sgd}. Then
  \[
  \ebb[\ibb[i_{b,j_{b,t}}=n]]=\frac{1}{n}.
  \]
\end{lemma}
\begin{proof}
  We know
  \[
    \ebb[\ibb[i_{b,j_{b,t}}=n]] =\pr\big\{i_{b,j_{b,t}}=n\big\} = \sum_{r=0}^{m}\pr\Big\{\sum_{k=1}^m\ibb[i_{b,k}=n]=r\Big\}\cdot\pr\Big\{i_{b,j_{b,t}}=n\big|\sum_{k=1}^m\ibb[i_{b,k}=n]=r\Big\}.
  \]
  It is clear that $\pr\Big\{i_{b,j_{b,t}}=n\big|\sum_{k=1}^m\ibb[i_{b,k}=n]=r\Big\}=r/m$. Therefore, we can apply Lemma \ref{lem:comb} to derive that
  \begin{align*}
  \ebb[\ibb[i_{b,j_{b,t}}=n]] &=  \frac{1}{mn^m}\sum_{r=1}^{m}rC_m^r(n-1)^{m-r}=\frac{1}{n^m}\sum_{r=1}^{m}C_{m-1}^{r-1}(n-1)^{m-r}\\
  & = \frac{1}{n^m}\sum_{r=0}^{m-1}C_{m-1}^r(n-1)^{m-r-1}=\frac{(1+n-1)^{m-1}}{n^m}=\frac{1}{n},
  \end{align*}
  where we have used the following identity
  \begin{equation}\label{exp-ind-1}
  C_m^r\cdot r=\frac{m!r}{r!(m-r)!}=\frac{(m-1)!m}{(r-1)!(m-r)!}=mC_{m-1}^{r-1}.
  \end{equation}
  The proof is completed.
\end{proof}

We are now in a position to prove Theorem~\ref{thm:sgd}. Our analysis with the $\ell_2$-argument stability uses the expectation-variance technique in \citep{lei2023batch}.
\begin{proof}[Proof of Theorem \ref{thm:sgd}]
  According to Lemma \ref{lem:sgd}, we know
  \[
  \ebb_A\big[\big\|\bw_{T+1}^{(b)}-\tilde{\bw}_{T+1}^{(b)}\big\|_2\big]\leq \sum_{t=1}^{T}\eta_t\ebb_A[\delta_{t,b}\ibb[i_{b,j_{b,t}}=n]],
  \]
  where we have used the fact that $\delta_{t,b}$ is independent of $i_{b,j_{b,t}}$. It then follows that
  \begin{align}
  & \ebb_A\big[\|A(S)-A(\widetilde{S})\|_2\big] = \ebb_A\Big[\Big\|\frac{1}{B}\sum_{b=1}^{B}\bw_{T+1}^{(b)}-\frac{1}{B}\sum_{b=1}^{B}\tilde{\bw}_{T+1}^{(b)}\Big\|_2\Big] \notag\\
  & \leq \frac{1}{B}\sum_{b=1}^{B}\ebb_A\big[\big\|\bw_{T+1}^{(b)}-\tilde{\bw}_{T+1}^{(b)}\big\|_2\big]
  \leq  \frac{1}{B}\sum_{b=1}^{B}\sum_{t=1}^{T}\eta_t\ebb_A[\delta_{t,b}\ibb[i_{b,j_{b,t}}=n]]=\sum_{t=1}^{T}\eta_t\ebb_A[\delta_{t,1}\ibb[i_{1,j_{1,t}}=n]].\label{sgd-1}
  \end{align}
  We now consider the $\ell_2$-model stability as follows
  \begin{align}
  \ebb_A\big[\|A(S)-A(\widetilde{S})\|_2^2\big] & = \ebb_A\Big[\Big\|\frac{1}{B}\sum_{b=1}^{B}\bw_{T+1}^{(b)}-\frac{1}{B}\sum_{b=1}^{B}\tilde{\bw}_{T+1}^{(b)}\Big\|_2^2\Big] \notag\\
  & \leq \frac{1}{B^2}\ebb_A\Big[\Big(\sum_{b=1}^{B}\|\bw_{T+1}^{(b)}-\tilde{\bw}_{T+1}^{(b)}\|_2\Big)^2\Big]
  = \frac{1}{B^2}\ebb_A\Big[\Big(\sum_{b=1}^{B}\sum_{t=1}^{T}\eta_t\delta_{t,b}\ibb[i_{b,j_{b,t}}=n]\Big)^2\Big]\notag \\
  & = \frac{1}{B^2}\ebb_A\Big[\sum_{b=1}^{B}\sum_{b'=1}^B\Big(\sum_{t=1}^{T}\eta_t\delta_{t,b}\ibb[i_{b,j_{b,t}}=n]\Big)\Big(\sum_{t=1}^{T}\eta_t\delta_{t,b'}\ibb[i_{b',j_{b',t}}=n]\Big)\Big]\notag\\
  & = \frac{1}{B^2}\ebb_A\Big[\sum_{b=1}^{B}\Big(\sum_{t=1}^{T}\eta_t\delta_{t,b}\ibb[i_{b,j_{b,t}}=n]\Big)\Big(\sum_{t=1}^{T}\eta_t\delta_{t,b}\ibb[i_{b,j_{b,t}}=n]\Big)\Big]\notag\\
  & + \frac{1}{B^2}\ebb_A\Big[\sum_{b\neq b'}\Big(\sum_{t=1}^{T}\eta_t\delta_{t,b}\ibb[i_{b,j_{b,t}}=n]\Big)\Big(\sum_{t=1}^{T}\eta_t\delta_{t,b'}\ibb[i_{b',j_{b',t}}=n]\Big)\Big].
  \end{align}
  For any $b\neq b'$ and $t, t'$, the  independence  between different bootstrap samples implies
  \begin{align}
    \ebb_A\big[\|A(S)-A(\widetilde{S})\|_2^2\big] &\leq \frac{1}{B}\ebb_A\Big[\Big(\sum_{t=1}^{T}\eta_t\delta_{t,1}\ibb[i_{1,j_{1,t}}=n]\Big)\Big(\sum_{t=1}^{T}\eta_t\delta_{t,1}\ibb[i_{1,j_{1,t}}=n]\Big)\Big]\notag\\
    & + \frac{1}{B^2}\sum_{b\neq b'}\ebb_A\Big[\Big(\sum_{t=1}^{T}\eta_t\delta_{t,b}\ibb[i_{b,j_{b,t}}=n]\Big)\Big]\ebb_A\Big[\Big(\sum_{t=1}^{T}\eta_t\delta_{t,b'}\ibb[i_{b',j_{b',t}}=n]\Big)\Big]\notag.
  \end{align}

  That is,
  \begin{align}
    \ebb_A\big[\|A(S)-A(\widetilde{S})\|_2^2\big] & \leq \frac{1}{B^2}\ebb_A\Big[\Big(\sum_{b=1}^{B}\|\bw_{T+1}^{(b)}-\tilde{\bw}_{T+1}^{(b)}\|_2\Big)^2\Big]\notag\\
    & \leq  \frac{1}{B}\ebb_A\Big[\Big(\sum_{t=1}^{T}\eta_t\delta_{t,1}\ibb[i_{1,j_{1,t}}=n]\Big)^2\Big]+\frac{B^2-B}{B^2}\Big(\ebb_A\Big[\sum_{t=1}^{T}\eta_t\delta_{t,1}\ibb[i_{1,j_{1,t}}=n]\Big]\Big)^2.\label{l2-stab-1}
  \end{align}

Using $(a+b)^{2}\le 2(a^{2}+b^{2})$ for $a,b\in\R$ 
we obtain
  \begin{align}
     & \ebb_A\Big[\Big(\sum_{t=1}^{T}\eta_t\delta_{t,1}\ibb[i_{1,j_{1,t}}=n]\Big)^2\Big] \notag\\
     & =\sum_{r=0}^{m}\pr\Big\{\sum_{k=1}^m\ibb[i_{1,k}=n]=r\Big\}\cdot\ebb_A\Big[\Big(\sum_{t=1}^{T}\eta_t\delta_{t,1}\ibb[i_{1,j_{1,t}}=n]\Big)^2\big|\sum_{k=1}^m\ibb[i_{1,k}=n]=r\Big]\notag\\
     & \leq 2\sum_{r=0}^{m}\pr\Big\{\sum_{k=1}^m\ibb[i_{1,k}=n]=r\Big\}\cdot\ebb_A\Big[\Big(\sum_{t=1}^{T}\eta_t\delta_{t,1}\big(\ibb[i_{1,j_{1,t}}=n]-r/m\big)\Big)^2\big|\sum_{k=1}^m\ibb[i_{1,k}=n]=r\Big]\notag\\
     & + 2\sum_{r=0}^{m}\pr\Big\{\sum_{k=1}^m\ibb[i_{1,k}=n]=r\Big\}\cdot\ebb_A\Big[\Big(\sum_{t=1}^{T}\eta_t\delta_{t,1}r/m\Big)^2\big|\sum_{k=1}^m\ibb[i_{1,k}=n]=r\Big].\label{l2-stab-2}
  \end{align}
  For any $t\neq t'$ (we assume $t<t'$), we know
  \begin{align*}
  & \ebb_A\big[\delta_{t,1}\big(\ibb[i_{1,j_{1,t}}=n]-r/m\big)\delta_{t',1}\big(\ibb[i_{1,j_{1,t'}}=n]-r/m\big)\big|\sum_{k=1}^m\ibb[i_{1,k}=n]=r\big]\\
  &=
  \ebb_A\big[\delta_{t,1}\big(\ibb[i_{1,j_{1,t}}=n]-r/m\big)\delta_{t',1}\ebb_{j_{1,t'}}\big[\big(\ibb[i_{1,j_{1,t'}}=n]-r/m\big)\big|\sum_{k=1}^m\ibb[i_{1,k}=n]=r\big]\big|\sum_{k=1}^m\ibb[i_{1,k}=n]=r\big]=0,
  \end{align*}
  where we have used the fact that $\ebb_{j_{1,t'}}\big[\big(\ibb[i_{1,j_{1,t'}}=n]-r/m\big)\big|\sum_{k=1}^m\ibb[i_{1,k}=n]=r\big]=0$. It then follows that
  \begin{align*}
  & \ebb_A\Big[\Big(\sum_{t=1}^{T}\eta_t\delta_{t,1}\big(\ibb[i_{1,j_{1,t}}=n]-r/m\big)\Big)^2\big|\sum_{k=1}^m\ibb[i_{1,k}=n]=r\Big]\\
  & = \ebb_A\Big[\sum_{t=1}^{T}\eta_t^2\delta_{t,1}^2\big(\ibb[i_{1,j_{1,t}}=n]-r/m\big)^2\big|\sum_{k=1}^m\ibb[i_{1,k}=n]=r\Big]\\
  & + \ebb_A\Big[\sum_{t\neq t'}\eta_t\eta_{t'}\delta_{t,1}\delta_{t',1}\big(\ibb[i_{1,j_{1,t}}=n]-r/m\big)\big(\ibb[i_{1,j_{1,t'}}=n]-r/m\big)\big|\sum_{k=1}^m\ibb[i_{1,k}=n]=r\Big]\\
  & = \ebb_A\Big[\sum_{t=1}^{T}\eta_t^2\delta_{t,1}^2\big(\ibb[i_{1,j_{1,t}}=n]-r/m\big)^2\big|\sum_{k=1}^m\ibb[i_{1,k}=n]=r\Big]
  \leq \ebb_A\Big[\sum_{t=1}^{T}\eta_t^2\delta_{t,1}^2(\ibb[i_{1,j_{1,t}}=n])^2\big|\sum_{k=1}^m\ibb[i_{1,k}=n]=r\Big]\\
  & = \sum_{t=1}^{T}\eta_t^2\ebb_A\big[\delta_{t,1}^2\ibb[i_{1,j_{1,t}}=n]\big|\sum_{k=1}^m\ibb[i_{1,k}=n]=r\big]= \frac{r}{m}\ebb_A\Big[\sum_{t=1}^{T}\eta_t^2\delta_{t,1}^2\big|\sum_{k=1}^m\ibb[i_{1,k}=n]=r\Big],
  \end{align*}
  where we have used the fact that $\ebb_{j_{1,t}}\big[\ibb[i_{1,j_{1,t}}=n]\big|\sum_{k=1}^m\ibb[i_{1,k}=n]=r\big]=r/m$ and $\ebb[(X-\ebb[X])^2]\leq\ebb[X^2]$ in the inequality.
  We can combine the above discussions to derive that
  \begin{align*}
     \ebb_A\Big[\Big(\sum_{t=1}^{T}\eta_t\delta_{t,1}\ibb[i_{1,j_{1,t}}=n]\Big)^2\Big]
     & \leq \frac{2}{m}\sum_{r=1}^{m}r\pr\Big\{\sum_{k=1}^m\ibb[i_{1,k}=n]=r\Big\}\cdot\ebb_A\Big[\sum_{t=1}^{T}\eta_t^2\delta_{t,1}^2\big|\sum_{k=1}^m\ibb[i_{1,k}=n]=r\Big]\\
     & + \frac{2}{m^2}\sum_{r=0}^{m}r^2\pr\Big\{\sum_{k=1}^m\ibb[i_{1,k}=n]=r\Big\}\cdot\ebb_A\Big[\Big(\sum_{t=1}^{T}\eta_t\delta_{t,1}\Big)^2\big|\sum_{k=1}^m\ibb[i_{1,k}=n]=r\Big].
  \end{align*}
  We plug the above inequality back into Eq. \eqref{l2-stab-1}, and derive
  \begin{align*}
    \ebb_A\big[\|A(S)-A(\widetilde{S})\|_2^2\big]
     & \leq \frac{2}{mB}\sum_{r=1}^{m}r\pr\Big\{\sum_{k=1}^m\ibb[i_{1,k}=n]=r\Big\}\cdot\ebb_A\Big[\sum_{t=1}^{T}\eta_t^2\delta_{t,1}^2\big|\sum_{k=1}^m\ibb[i_{1,k}=n]=r\Big]\\
  & +\frac{2}{m^2B}\sum_{r=0}^{m}r^2\pr\Big\{\sum_{k=1}^m\ibb[i_{1,k}=n]=r\Big\}\cdot\ebb_A\Big[\Big(\sum_{t=1}^{T}\eta_t\delta_{t,1}\Big)^2\big|\sum_{k=1}^m\ibb[i_{1,k}=n]=r\Big]\\
  & +\frac{B^2-B}{B^2}\ebb_A\Big[\Big(\ebb_A\Big[\sum_{t=1}^{T}\eta_t\delta_{t,1}\ibb[i_{1,j_{1,t}}=n]\Big]\Big)^2\Big].
  \end{align*}
  The proof is completed by using Lemma \ref{lem:comb} and Jensen's inequality.
\end{proof}
\begin{proof}[Proof of Corollary \ref{cor:sgd}]
  Since $\bw\mapsto f(\bw;\bz)$ is $G$-Lipschitz continuous, we know $\delta_{t,b}\leq 2G$. It then follows from Theorem \ref{thm:sgd} that
  \[
  \ebb_A\big[\|A(S)-A(\widetilde{S})\|_2\big]\leq 2G\sum_{t=1}^{T}\eta_t\ebb_A[\ibb[i_{1,j_{1,t}}=n]]=\frac{2G}{n}\sum_{t=1}^{T}\eta_t,
  \]
  where we have used Lemma \ref{lem:exp-indicator}.

  Since $\delta_{t,b}\leq 2G$, Theorem \ref{thm:sgd} implies
  \begin{multline*}
    \ebb_A\big[\|A(S)-A(\widetilde{S})\|_2^2\big]
     \leq \frac{8G^2}{mB}\sum_{r=1}^{m}\frac{rC_m^r(n-1)^{m-r}}{n^m}\cdot\ebb_A\Big[\sum_{t=1}^{T}\eta_t^2\big|\sum_{k=1}^m\ibb[i_{1,k}=n]=r\Big]+\\
  \frac{8G^2}{m^2B}\sum_{r=0}^{m}\frac{r^2C_m^r(n-1)^{m-r}}{n^m}\cdot\ebb_A\Big[\Big(\sum_{t=1}^{T}\eta_t\Big)^2\big|\sum_{k=1}^m\ibb[i_{1,k}=n]=r\Big]+\frac{8G^2(B-1)}{B}\Big(\ebb_A\Big[\sum_{t=1}^{T}\eta_t\ibb[i_{1,j_{1,t}}=n]\Big]\Big)^2.
  \end{multline*}
  By the identity $\sum_{r=0}^{m}\frac{rC_m^r(n-1)^{m-r}}{mn^m}=\frac{1}{n}$ shown in the proof of Lemma \ref{lem:exp-indicator}, we further know
  \begin{multline*}
    \ebb_A\big[\|A(S)-A(\widetilde{S})\|_2^2\big]
     \leq \frac{8G^2}{Bn}\sum_{t=1}^{T}\eta_t^2+
  \frac{8G^2}{m^2B}\sum_{r=0}^{m}\frac{r^2C_m^r(n-1)^{m-r}}{n^m}\Big(\sum_{t=1}^{T}\eta_t\Big)^2+\frac{8G^2(B-1)}{Bn^2}\Big(\sum_{t=1}^{T}\eta_t\Big)^2.
  \end{multline*}
  Furthermore, we know
  \begin{align*}
    & \sum_{r=0}^{m}\frac{r^2C_m^r(n-1)^{m-r}}{m^2n^m} \leq \frac{1}{nm}+ \sum_{r=2}^{m}\frac{rC_{m-1}^{r-1}(n-1)^{m-r}}{mn^m}\\
    & = \frac{1}{nm}+ 2\sum_{r=2}^{m}\frac{C_{m-2}^{r-2}(m-1)(n-1)^{m-r}}{mn^m}\leq \frac{1}{nm}+ 2\sum_{r=0}^{m-2}\frac{C_{m-2}^{r}(n-1)^{m-r-2}}{n^m}\\
    & = \frac{1}{nm}  + \frac{2(1+n-1)^{m-2}}{n^m} = \frac{1}{nm}+\frac{2}{n^2},
  \end{align*}
  where we have used Eq. \eqref{exp-ind-1} and the following inequality for any $r\geq 2$
  \[
  \frac{(m-1)!r}{(r-1)!(m-r)!}\leq \frac{2(m-1)!}{(r-2)!(m-r)!}=\frac{2(m-2)!(m-1)}{(r-2)!(m-r)!}=2C_{m-2}^{r-2}(m-1).
  \]
  We then get that
    \[
    \ebb_A\big[\|A(S)-A(\widetilde{S})\|_2^2\big]
     \leq \frac{8G^2}{Bn}\sum_{t=1}^{T}\eta_t^2+
     \frac{8G^2}{Bn}\Big(\frac{2}{n}+\frac{1}{m}\Big)\Big(\sum_{t=1}^{T}\eta_t\Big)^2
     + \frac{8G^2(B-1)}{n^2B}\Big(\sum_{t=1}^{T}\eta_t\Big)^2.
  \]
  The proof is completed.
\end{proof}

\subsection{Proofs for Bootstrap SGD of Type 2\label{sec:proof-type2}}

\begin{proof}[Proof of Theorem \ref{thm:stab-type2}]
According to the definition of $h_S^{(2)}$ and $h^{(2)}_{\widetilde{S}}$ and the $G_h$-Lipschitz continuity of $h_{\bw}$, we know
\begin{align*}
  \big|h^{(2)}_S(\bx)-h^{(2)}_{\widetilde{S}}(\bx)\big| & = \frac{1}{B}\Big|\sum_{b=1}^{B}\big(h_{\bw_{T+1}^{(b)}}(\bx)-h_{\tilde{\bw}_{T+1}^{(b)}}(\bx)\big)\Big|
   \leq \frac{G_h}{B}\sum_{b=1}^{B}\|\bw_{T+1}^{(b)}-\tilde{\bw}_{T+1}^{(b)}\|_2.
\end{align*}
It then follows that
\[
\ebb_A\big[\sup_{\bx}\big|h^{(2)}_S(\bx)-h^{(2)}_{\widetilde{S}}(\bx)\big|\big] \leq \frac{G_h}{B}\sum_{b=1}^{B}\ebb_A[\|\bw_{T+1}^{(b)}-\tilde{\bw}_{T+1}^{(b)}\|_2]\leq
G_h\sum_{t=1}^{T}\eta_t\ebb_A[\delta_{t,1}\ibb[i_{1,j_{1,t}}=n]],
\]
where the last step follows from Eq. \eqref{sgd-1}.

We now consider the $\ell_2$-stability. Similarly, we have
\[
  \big(h^{(2)}_S(\bx)-h^{(2)}_{\widetilde{S}}(\bx)\big)^2 = \Big(\frac{1}{B}\sum_{b=1}^{B}\big(h_{\bw_{T+1}^{(b)}}(\bx)-h_{\tilde{\bw}_{T+1}^{(b)}}(\bx)\big)\Big)^2
   \leq \Big(\frac{G_h}{B}\sum_{b=1}^{B}\|\bw_{T+1}^{(b)}-\bw_{T+1}^{(b)}\|_2\Big)^2.
\]
It then follows from Eq. \eqref{l2-stab-1}, Lemma~\ref{lem:comb} and the proof of Theorem \ref{thm:sgd} that
\begin{multline*}
  \ebb_A\big[\sup_{\bx}\big(h^{(2)}_S(\bx)-h^{(2)}_{\widetilde{S}}(\bx)\big)^2\big] \leq
  \frac{2G_h^2}{mB}\sum_{r=1}^{m}\frac{rC_m^r(n-1)^{m-r}}{n^m}\cdot\ebb_A\Big[\sum_{t=1}^{T}\eta_t^2\delta_{t,1}^2\big|\sum_{k=1}^m\ibb[i_{1,k}=n]=r\Big]+\\
  \frac{2G_h^2}{m^2B}\sum_{r=0}^{m}\frac{r^2C_m^r(n-1)^{m-r}}{n^m}\cdot\ebb_A\Big[\Big(\sum_{t=1}^{T}\eta_t\delta_{t,1}\Big)^2\big|\sum_{k=1}^m\ibb[i_{1,k}=n]=r\Big]+ G_h^2(1-\frac{1}{B})\Big(\ebb_A\Big[\sum_{t=1}^{T}\eta_t\delta_{t,1}\ibb[i_{1,j_{1,t}}=n]\Big]\Big)^2.
\end{multline*}
The proof is completed.
\end{proof}

Before proving Corollary \ref{cor:gen-type2}, we first present a lemma relating stability and generalization for a randomized algorithm. It is a variant of Lemma \ref{lem:stab-gen}.
Let $S=\{z_1,\ldots,z_n\}$ and $S'=\{z_1',\ldots,z_n'\}$ be independently drawn from the same distribution. For any $i\in[n]$, define
\begin{equation}\label{S-i}
  S_i=\{z_1,\ldots,z_{i-1},z_i',z_{i+1},\ldots,z_n\}.
\end{equation}
\begin{lemma}\label{lem:stab-gen-ell}
For any $i\in[n]$, let $S_i$ be defined in Eq. \eqref{S-i}. Let $h_S$ be the output of an algorithm applied to $S$.  Then we have
\[
\ebb_A[\ell(h_S(x),y)]-\frac{1}{n}\sum_{i=1}^{n}\ebb_A[\ell(h_S(x_i),y_i)]\leq \ebb_A\big[\sup_{z}\max_{i\in[n]}\big(\ell(h_{S_i}(x),y)-\ell(h_S(x),y)\big)\big].
\]
Furthermore, if for any $y$, the function $a\mapsto\ell(a,y)$ is $L_\ell$-smooth, then
\begin{multline*}
  \ebb_A[\ell(h_S(x),y)]-\frac{1}{n}\sum_{i=1}^{n}\ebb_A[\ell(h_S(x_i),y_i)]\leq  \\
  \Big(\frac{1}{n}\sum_{i=1}^{n}\ebb_A[(\ell'(h_S(x_i),y_i))^2]\Big)^{\frac{1}{2}}\Big(\frac{1}{n}\sum_{i=1}^{n}\ebb_A\big[\big(h_{S_i}(x_i)-h_{S}(x_i)\big)^2\big]\Big)^{\frac{1}{2}}+\frac{L_\ell}{2n}\sum_{i=1}^{n}\ebb_A\big[\big(h_{S_i}(x_i)-h_{S}(x_i)\big)^2\big].
\end{multline*}
\end{lemma}
\begin{proof}
Due to the symmetry between $z_i$ and $z_i'$, we know
\begin{align*}
  & \ebb_A[\ell(h_S(x),y)]-\frac{1}{n}\sum_{i=1}^{n}\ebb_A[\ell(h_S(x_i),y_i)]  = \frac{1}{n}\sum_{i=1}^{n}\ebb_A[\ell(h_{S_i}(x),y)]-\frac{1}{n}\sum_{i=1}^{n}\ebb_A[\ell(h_S(x_i),y_i)] \\
  & = \frac{1}{n}\sum_{i=1}^{n}\ebb_A[\ell(h_{S_i}(x_i),y_i)-\ell(h_S(x_i),y_i)]
   \leq \ebb_A\big[\sup_{z}\max_{i\in[n]}\big(\ell(h_{S_i}(x),y)-\ell(h_S(x),y)\big)\big],
\end{align*}
where we have used the fact that $(x_i,y_i)$ is independent of $h_{S_i}$.
This finishes the proof of the first inequality. We now turn to the smooth case. By using the smoothness property in the above inequality, we know
\begin{align*}
  & \ebb_A[\ell(h_S(x),y)]-\frac{1}{n}\sum_{i=1}^{n}\ebb_A[\ell(h_S(x_i),y_i)]\\
  & \leq \frac{1}{n}\sum_{i=1}^{n}\ebb_A\big[\ell'(h_S(x_i),y_i)\big(h_{S_i}(x_i)-h_{S}(x_i)\big)\big]+\frac{L_\ell}{2n}\sum_{i=1}^{n}\ebb_A\big[\big(h_{S_i}(x_i)-h_{S}(x_i)\big)^2\big]\\
  & \leq \ebb_A\Big(\frac{1}{n}\sum_{i=1}^{n}(\ell'(h_S(x_i),y_i))^2\Big)^{\frac{1}{2}}\Big(\frac{1}{n}\sum_{i=1}^{n}\big(h_{S_i}(x_i)-h_{S}(x_i)\big)^2\Big)^{\frac{1}{2}}+\frac{L_\ell}{2n}\sum_{i=1}^{n}\ebb_A\big[\big(h_{S_i}(x_i)-h_{S}(x_i)\big)^2\big]\\
  & \leq
  \Big(\frac{1}{n}\sum_{i=1}^{n}\ebb_A[(\ell'(h_S(x_i),y_i))^2]\Big)^{\frac{1}{2}}\Big(\frac{1}{n}\sum_{i=1}^{n}\ebb_A\big[\big(h_{S_i}(x_i)-h_{S}(x_i)\big)^2\big]\Big)^{\frac{1}{2}}+\frac{L_\ell}{2n}\sum_{i=1}^{n}\ebb_A\big[\big(h_{S_i}(x_i)-h_{S}(x_i)\big)^2\big],
\end{align*}
where we have used the Cauchy's inequality. The stated bound follows directly.
\end{proof}
\begin{proof}[Proof of Corollary \ref{cor:gen-type2}]
Since $a\mapsto\ell(a,y)$ is $G_\ell$-Lipschitz continuous,
Lemma \ref{lem:stab-gen-ell} implies
\begin{equation}\label{gen-type2-1}
  \ebb_A[\ell(h_S^{(2)}(x),y)]-\frac{1}{n}\sum_{i=1}^{n}\ebb_A[\ell(h_S^{(2)}(x_i),y_i)]\leq  G_\ell\ebb_A\big[\sup_{x}\max_{i\in[n]}\big|h^{(2)}_{S_i}(x)-h_S^{(2)}(x)\big|\big]
\end{equation}
and
\begin{multline}\label{gen-type2-2}
  \ebb_A[\ell(h_S^{(2)}(x),y)]-\frac{1}{n}\sum_{i=1}^{n}\ebb_A[\ell(h_S^{(2)}(x_i),y_i)]\leq  \\
  G_\ell\Big(\ebb_A\big[\sup_{x}\big(h^{(2)}_{S_i}(x)-h^{(2)}_{S}(x)\big)^2\big]\Big)^{\frac{1}{2}}+2^{-1}L_\ell\ebb_A\big[\sup_{x}\big(h^{(2)}_{S_i}(x)-h^{(2)}_{S}(x)\big)^2\big].
\end{multline}
Eq. \eqref{gen-type2-1} together with Eq. \eqref{stab-type2-a} implies that
\[
\ebb_A[\ell(h_S^{(2)}(x),y)]-\frac{1}{n}\sum_{i=1}^{n}\ebb_A[\ell(h_S^{(2)}(x_i),y_i)]\leq 2GG_hG_\ell\sum_{t=1}^{T}\eta_t\ebb_A[\ibb[i_{1,j_{1,t}}=n]]\leq \frac{2GG_hG_\ell}{n}\sum_{t=1}^{T}\eta_t,
\]
where the last step follows from Lemma \ref{lem:exp-indicator}.

Furthermore, Theorem \ref{thm:stab-type2} implies that
\begin{multline*}
  \ebb_A\big[\sup_{\bx}\big(h^{(2)}_S(\bx)-h^{(2)}_{\widetilde{S}}(\bx)\big)^2\big] \leq
  \frac{8G^2G_h^2}{mB}\sum_{r=1}^{m}\frac{rC_m^r(n-1)^{m-r}}{n^m}\sum_{t=1}^{T}\eta_t^2+\\
  \frac{8G^2G_h^2}{m^2B}\sum_{r=0}^{m}\frac{r^2C_m^r(n-1)^{m-r}}{n^m}\Big(\sum_{t=1}^{T}\eta_t\Big)^2+ 4G^2G_h^2(1-\frac{1}{B})\Big(\ebb_A\Big[\sum_{t=1}^{T}\eta_t\ibb[i_{1,j_{1,t}}=n]\Big]\Big)^2.
\end{multline*}
Analogous to the proof of Corollary \ref{cor:sgd}, we know that
\[
\ebb_A\big[\sup_{\bx}\big(h^{(2)}_S(\bx)-h^{(2)}_{\widetilde{S}}(\bx)\big)^2\big] \lesssim \frac{G^2G_h^2}{Bn}\sum_{t=1}^{T}\eta_t^2+
     G^2G_h^2\Big(\frac{1}{Bmn}+\frac{1}{n^2}\Big)\Big(\sum_{t=1}^{T}\eta_t\Big)^2.
\]
We plug the above inequality back into Eq. \eqref{gen-type2-2}, and derive
\begin{multline*}
  \ebb_A[\ell(h_S^{(2)}(x),y)]-\frac{1}{n}\sum_{i=1}^{n}\ebb_A[\ell(h_S^{(2)}(x_i),y_i)] \lesssim G_\ell GG_h\bigg(\frac{1}{Bn}\sum_{t=1}^{T}\eta_t^2+
     \Big(\frac{1}{Bmn}+\frac{1}{n^2}\Big)\Big(\sum_{t=1}^{T}\eta_t\Big)^2\bigg)^{\frac{1}{2}}
     +\\
     \frac{G^2G_h^2L_\ell}{Bn}\sum_{t=1}^{T}\eta_t^2+
     G^2G_h^2L_\ell\Big(\frac{1}{Bmn}+\frac{1}{n^2}\Big)\Big(\sum_{t=1}^{T}\eta_t\Big)^2.
\end{multline*}
The proof is completed.
\end{proof}



\setlength{\bibsep}{0.04cm}
\bibliographystyle{abbrvnat}
\small

\end{document}